\documentclass[conference,letterpaper]{IEEEtran}

\usepackage[utf8]{inputenc} 
\usepackage[T1]{fontenc}
\usepackage{url}
\usepackage{ifthen}

\usepackage[cmex10]{amsmath} 

\usepackage{graphicx}

\usepackage{tikz}
\usetikzlibrary{positioning, calc}

\usepackage{amsmath,amssymb,amsfonts}
\usepackage{amsthm}

\newcommand{\R}{\mathbb{R}}
\newcommand{\E}{\mathbb{E}}

\usepackage{xcolor}

\newtheorem{theorem}{Theorem}
\newtheorem{lemma}[theorem]{Lemma}
\newtheorem{proposition}[theorem]{Proposition}
\newtheorem{corollary}[theorem]{Corollary}
\newtheorem{definition}[theorem]{Definition}
\newtheorem{remark}[theorem]{Remark}

\newtheorem{assumption}[theorem]{Assumption}

\begin{document}
\title{A Hierarchical Sampling Framework for bounding the Generalization Error of Federated Learning}

\author{
  \IEEEauthorblockN{Dario Filatrella, Ragnar Thobaben, and Mikael Skoglund}
  \IEEEauthorblockA{\\
                    Department of Information Science and Engineering \\
                    Royal Institute of Technology \\
                    Stockholm, Sweden \\
                    Email: \{dariofi, ragnart, skoglund\}@kth.se}
}

\maketitle

\begin{abstract}
    We study expected generalization bounds for the Hierarchical Federated Learning (HFL) setup using Wasserstein distance. 
    We introduce a generalized framework in which data is sampled hierarchically, and we model it with a multi-layered tree structure that induces dependencies among the clients' datasets.
    We derive generalization bounds in terms of Wasserstein distance under the Lipschitz assumption on the loss function, by applying a supersample construction that allows us to measure
    the sensitivity of the algorithm to the change of a single node in the sampling tree. By leveraging the FL structure, we recover and strictly imply existing state-of-the-art conditional
    mutual information (CMI) bounds in the case of bounded losses. We also show that our bound can be applied together with Differential Privacy assumptions, to recover 
    generalization bounds based on algorithmic privacy.
    To assess the tightness of our bounds, we study the Gaussian Location Model (GLM) and show that we recover the actual asymptotic rate of the generalization error .
\end{abstract}

\section{Introduction}

Federated Learning (FL) is a machine learning setting where multiple clients (or nodes) collaboratively train a global model \cite{kairouzAdvancesOpenProblems2021}. 
The term was introduced in 2016 by \cite{mcmahanCommunicationEfficientLearningDeep2023}, where the authors present the challenges and motivations of training models in a distributed manner, 
and propose the FedAvg algorithm.

The two main motivations to train models in a distributed fashion are: (i) to distribute training resources for models that are exceptionally large and (ii) to distribute
datasets in the case of privacy concerns or communication constraints.

Motivation (i) is an engineering driven motivation, and for further details, we refer to \cite{zhangDeepLearningElastic2015} and \cite{poveyParallelTrainingDNNs2015}, 
which provide frameworks for training DNNs in a distributed manner. 
Motivation (ii) is what we aim to address with the random participation of clients. 
The idea of having a complex hierarchical data source, with non i.i.d. points, first appeared in classical statistics, specifically in the field of sampling techniques, when, in 1965, Kish introduced the concept of \textit{design effect} 
and \textit{effective sample size} \cite{kish1965survey}. There exist many classical references on sampling techniques, such as \cite{cochranSamplingTechniques1977}.
These results predate the modern learning theory framework, and aim to measure the heterogeneity of the data through variance-like quantities 
and indicators, such as the \textit{intraclass correlation coefficient}, which are not compatible with modern learning theory approaches based on complexity measures and information theory.

While generalization bounds for supervised learning (SL) could be applied, as the FL setup can be seen as a "black box" SL setup, 
they would lose the leverage of the FL structure, and, in some cases, they might also incur in clashing assumptions due to the non-i.i.d. nature of the data. 
Despite those reasons calling for the development of generalization bounds tailored for FL, the literature 
on this topic remains in its early stages.

In our work, we primarily address motivation (ii), leaving flexibility to incorporate constraints about the computational power of each node.
Our main contributions are: 
\begin{itemize}
    \item A generalized framework that encompasses hierarchical sampling strategies and previous Federated Learning frameworks as special cases.
    \item A generalization bound in terms of Wasserstein distance that holds for a wide range of loss functions and, in particular, is tighter than previous state-of-the-art CMI bounds for bounded losses.
    \item A generalization bound under privacy assumptions on the local algorithms.
    \item A case study of the Gaussian Location Model (GLM), where we show how our bound compares to the known generalization error.
\end{itemize}

The remainder of the paper is organized as follows:.
Section \ref{section:related_work} reviews the related literature on FL and generalization bounds. Section \ref{section:preliminaries} 
introduces the notation and our framework. Section \ref{section:main_results} presents our main results, 
including both Wasserstein distance bounds, Differential Privacy bounds, and a discussion on the implications of our results. 
Then, Section \ref{section:gaussian_location_model} presents the case study of the GLM.
Finally, we include the proofs of every result and additional discussions in the appendix.

\section{Related work}
\label{section:related_work}

\subsection{Generalization theory}

The literature for generalization bounds in the centralized setting is extensive, and dates back to the 70s and 80s
starting with the PAC learning framework \cite{valiantTheoryLearnable1984} and the Vapnik-Chervonenkis dimension \cite{vapnik2015uniform}.
Those approaches are based on the complexity of the hypothesis class, and, thanks to \cite{zhangUnderstandingDeepLearning2017},
 we know that those classical complexity measures are not sufficient to explain the generalization of modern deep learning models, where the number of parameters is often much larger than the number of training samples
 and the learning algorithms play a key role in the generalization performance.
Subsequently, a new theory based on information measures was developed, 
with results such as \cite{xuInformationtheoreticAnalysisGeneralization2017}, where the authors established a bound in terms of information measures between the data and the output hypothesis of the learning algorithm,
and has been improved numerous times, for example, in \cite{buTighteningMutualInformation2020}.

In 2020 \cite{steinkeReasoningGeneralizationConditional} introduced the conditional mutual information (CMI) framework and the supersample construction, which we build upon in this work.
This approach has yielded tighter bounds and has been extended to a larger class of losses using $f$-divergences and $phi$-divergences
\cite{harutyunyanInformationtheoreticGeneralizationBounds2021} and \cite{hellstromGeneralizationBoundsInformation2020}.

Of particular interest for our work is \cite{rodriguez-galvezTighterExpectedGeneralization2022} which introduced a new approach
based on Wasserstein distance, which implies the CMI bound as a special case for bounded losses.
The Wasserstein distance is particularly interesting for its connection with optimal transportation and its ability to capture the geometry of the distributions
\cite{santambrogioOptimalTransportApplied}.

\subsection{Federated Learning}

Due to the recency of FL, with the first ideas of training local models and averaging them appearing in \cite{mcdonaldDistributedTrainingStrategies} for the perceptron model,
studies of the generalization performance of FL algorithms started to emerge in the literature only recently \cite{yuanWhatWeMean2022}, with the introduction of the concept of 
\textbf{participation gap} and \textbf{Out-of-sample gap}, which provide the foundations for our contributions. 
In this early work, both theoretical bounds and empirical measurements are provided for the two gaps, using common image classification datasets.
Following this work, \cite{kimFedHBHierarchicalBayesian2023} introduces a hierarchical Bayesian structure with personalized models for each
client. They prove a convergence rate and a generalization bound for a locally convex model, with the latter being a very strong limitation.

An important aspect of FL is the communication structure, often summarized by the number of communication rounds (i.e. the number of times the nodes are allowed to exchange information).
In the special case of the \textit{single round} or \textit{one-shot} setting, it is possible to derive specialized generalization bounds \cite{peiReviewFederatedLearning2024,zhangImprovingGeneralizationFederated}.

In \cite{barnesImprovedInformationTheoretic2022} the authors introduce a mutual information bound for generic FL algorithms and then specialize it
to the case of models that can be represented as Bregman divergences, including a bound with communication constraints and multi-round algorithms.

In \cite{kavianHeterogeneityMattersEven2025} the authors show that the generalization errors is lower when clients data sources are more heterogeneous, in the case of averaging local models, 
but their limitations are related to the lack of leverage of the FL structure, which we discuss further when showing how to recover their bounds as a special case of our framework.

Overall, there is a lack of unified frameworks for FL that can ecompass different sampling strategies and communication structures.

\section{Preliminaries}
\label{section:preliminaries}

\subsection{Notation}
We use capital letters for random variables (e.g., $X$) and lowercase letters for their realizations (e.g., $x$).
We denote the distribution of a random variable $X$ as $P_X$, and the conditional distribution of $Y$ given $X$ as $P_{Y|X}$. 
We denote the Wasserstein-1 distance between two distributions $P$ and $Q$ as $\mathbb{W}(P, Q)$ (unless specified otherwise, we will always refer to the Wasserstein-1 distance).
We use the notation $i_{a:b}$ to denote the sequence of indices from $a$ to $b$, i.e., $i_a, i_{a+1}, \ldots, i_b$. 
When writing $\sum_{i_l^{1:l}}$ we mean the sum over all valid indices of "length" $l$, as understood from context.
Whenever $x_i$ is a sequence of object, we denote $x$ to be the ordered collection of all the objects and $|x|$ to be the number of objects in the sequence.
When writing $P \circ Q$ we mean the composition of the kernels, while $P \otimes Q$ means the product distribution. We use $P^{\otimes n}$ to denote the product distribution of $n$ independent copies of $P$.

\subsection{Hierarchical Federated Learning setup}

The Hierarchical Federated Learning setup is a supervised learning problem, where $\mathcal{Z}$ is the data space, $\mathcal{W}$ is the hypothesis space, 
and $\ell: \mathcal{W} \times \mathcal{Z} \to \R$ is the loss function.
We denote the dataset by $\mu$ and the learning algorithm by $P_{W|\mu}$, which is a (not necessarily) random kernel that maps the dataset to a distribution over hypotheses.
The federated nature of the setup is modeled by (i) the sampling structure of the dataset $\mu$ and (ii) the assumptions on the algorithm $P_{W|\mu}$.
In particular, our dataset $\mu$ is indexed by a tree with a fixed topology that represents a hierarchical sampling strategy, hence the name.
Given the constraints, we can index the nodes of the tree by their depth and position in the tree as follows: each node is indexed as 
$\mu_l^{i_{1:l}}$ where $l$ is the depth of the node and $i_{1:l}$ is the position of the node in the tree as a sequence of positions
of its ancestors (i.e. $i_{1}, i_{1:2}, \ldots, i_{1:l-1}$ are the ancestors of $i_{1:l}$). We also define $L$ to be the depth of 
the tree.

\begin{figure}[ht]
    \centering
     \resizebox{\columnwidth}{!}{%
    \begin{tikzpicture}[
    grow=down,
    level 1/.style={level distance=16mm, sibling distance=70mm},
    level 2/.style={level distance=18mm, sibling distance=16mm},
    every node/.style={draw, circle, inner sep=1.2pt, minimum size=8mm, font=\small},
    dot/.style={draw=none, circle=none, font=\small},
    edge from parent/.style={draw},
    ]
    \node (D) {$D$}
    child { node (m11) {$\mu_1^{1}$}
        child { node (m211) {$\mu_2^{1,1}$} }
        child { node (m212) {$\mu_2^{1,2}$} }
        child { node (m21n) {$\mu_2^{1,n_2}$} }
    }
    child { node (m12) {$\mu_1^{2}$}
        child { node (m221) {$\mu_2^{2,1}$} }
        child { node (m222) {$\mu_2^{2,2}$} }
        child { node (m22n) {$\mu_2^{2,n_2}$} }
    }
    child { node (m1n) {$\mu_1^{n_1}$}
        child { node[dot] {$\vdots$} } 
    };

    \path (D) -- (m12) coordinate[pos=0.55] (e12);
    \path (D) -- (m1n) coordinate[pos=0.55] (e1n);
    \node[dot] at ($(e12)!0.5!(e1n)$) {$\cdots$};

    \path (m11) -- (m212) coordinate[pos=0.55] (f12);
    \path (m11) -- (m21n) coordinate[pos=0.55] (f1n);
    \node[dot] at ($(f12)!0.5!(f1n)$) {$\cdots$};

    \path (m12) -- (m222) coordinate[pos=0.55] (g22);
    \path (m12) -- (m22n) coordinate[pos=0.55] (g2n);
    \node[dot] at ($(g22)!0.5!(g2n)$) {$\cdots$};

    \end{tikzpicture}
    }
    \caption{Hierarchical sampling structure}
\end{figure}
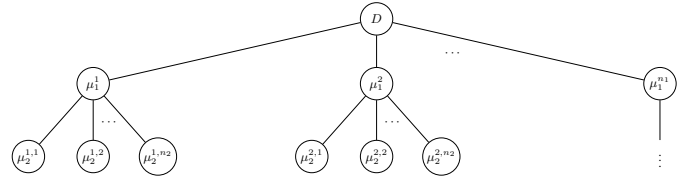

Formally the sampling is as follows:
the first layer's nodes are sampled i.i.d. from a meta-distribution $D$, meaning $\mu_1^{i_1} \sim D$ for all $i_1$. For the lower layers we have 
\begin{equation}
    \mu_l^{i_{1:l}} \sim \mu_{l-1}^{i_{1:l-1}}.
\end{equation}

Consequently, $\mu_L^{i_{1:L}} \in \mathcal{Z}$ is a data point, $\mu_{L-1}^{i_{1:L-1}}$ is a distribution over data points, and so on.
We can then write the sampling of a single data point as $\mu_L \sim \mu_{L-1} \sim \cdots \sim \mu_1 \sim D$, where $\mu_l$ is a random variable that takes values in the set of nodes at depth $l$.
For simplicity we will also shorten the previous notation to $\mu_L \sim \mu_1$ to indicate the whole chain of sampling from $L$ to $1$.

The intuitive motivation behind this dataset construction is the following: the leaves represent the actual data points, while the intermediate nodes represent 
clients or clusters, which are also assumed to be random, and in this setting they are indentified as distributions over the lower nodes (i.e. a cluster is a distribution over clients, a client is a distribution over data points).

Regarding the algorithmic assumptions, our main results hold for a general kernel $P_{W|\mu}$ that takes as input the whole tree without imposing any constraint (such as privacy mechanisms or communication constraints).
Of course, such bounds can only get tighter when restricting to a smaller class of algorithms, but we are still able to recover the bounds of \cite{wangGeneralizationFederatedLearning2025} and \cite{kavianHeterogeneityMattersEven2025}.
Their constraints consist of imposing each client to only communicate an hypothesis $W_l^{i_{1:l}} \in \mathcal{W}$ to the parent node, which will
then only act as an \textit{aggregation rule}, which can be modeled by a Markov chain.
In section \ref{section:privacy} we will explore such setting, and derive generalization bounds under privacy assumptions.

For technical reasons discussed in the proof of Theorem \ref{prop:gen_bound_decomposition}, we also assume constant branching factor at each layer,
i.e. we assume that at layer $l$ each node $\mu_l^{i_{1:l}}$ has $n_{l+1}$ children $\mu_{l+1}^{i_{1:l+1}}$. We can then denote $N_l = \prod_{k=1}^l n_k$ as the number of nodes at depth $l$.

\begin{remark}
    \label{remark:independence_assumption_middle_nodes}
    While the hypothesis $W$ clearly depends on the intermediate nodes $\mu_l^{i_{1:l}}$ for $l < L$, those nodes are usually an abstraction (such as the type of client) and they are
    not real data points that are used by the algorithm, which only uses leaf nodes.
    We can then model this with the following independence assumption:
    \begin{equation}
        W \perp \mu_l^{i_{1:l}} \mid \left\{ \mu_{l+1}^{i_{1:l+1}} \right\}_{i_{l+1}^{l+1} = 1}^{n_{l+1}} .
    \end{equation}
    We do not require this for our main results, but it is often implicitely assumed in the literature, for example in \cite{kavianHeterogeneityMattersEven2025,zhangImprovingGeneralizationFederated}.
\end{remark}

\begin{remark}
    \label{remark:positionally_dependant_sampling}
    With a slight modification we can assume that the sampling is dependent on the indices of the nodes, i.e., $\mu_l^{i_{1:l}} \sim K(\mu_{l-1}^{i_{1:l-1}}, i_{1:l})$, where $K$ is a kernel.
    In this case, the fixed depth becomes a non-real constraint, since we can just impose a certain node to output a deterministic copy until a leaf node is reached.
    Having positionally dependent kernel will make it possible to recover results from \cite{kavianHeterogeneityMattersEven2025}, but it is reasonable to assume symmetry among the nodes and 
    invariance of the learning algorithm to the ordering of the data points.
\end{remark}

\subsection{Generalization Error}

Given a realization of the dataset $\mu$, the empirical risk of a hypothesis $w$ is defined as
\begin{equation}
    \hat{L}(w, \mu) = \frac{1}{N_L} \sum_{i_{1:L}} \ell(w, \mu_L^{i_{1:L}}),
\end{equation}
and the population risk is defined as
\begin{equation}
    L(w) = \E_{\mu_L \sim \mu_{L-1} \sim \cdots \sim \mu_1 \sim D}[\ell(w, \mu_L)].
\end{equation}
The generalization error of an algorithm $P_{W|\mu}$ is defined then as:
\begin{equation}
    \label{eq:gen_error_definition}
    \text{gen}(P_{W|\mu}) = \E_{W \sim P_{W|\mu}, \mu_L \sim \mu_{L-1} \sim \cdots \sim \mu_1 \sim D}[L(W) - \hat{L}(W, \mu)].
\end{equation}

\begin{proposition}
    \label{prop:gen_bound_decomposition}
    [Generalization bound decomposition]
    Under the above setup, the generalization error can be decomposed as

\begin{IEEEeqnarray}{rCl}
\Delta_{l,i_{1:l}}
&&:=
\E_{W,\mu}\E_{\mu_L^{\mathrm{test}} \sim \mu_l^{i_{1:l}}}
\!\left[\ell\!\left(W,\mu_L^{\mathrm{test}}\right)\right]
\nonumber\\
&&\quad
- \frac{1}{n_{l+1}} \sum_{i_{l+1}}
\E_{W,\mu}\E_{\mu_L^{\mathrm{test}} \sim \mu_{l+1}^{i_{1:l+1}}}
\!\left[\ell\!\left(W,\mu_L^{\mathrm{test}}\right)\right],
\end{IEEEeqnarray}

\begin{IEEEeqnarray}{rCl}
\bigl|\operatorname{gen}(P_{W|\mu})\bigr|
&\leq&
\sum_{l=0}^{L-1} \frac{1}{N_l} \sum_{i_{1:l}}
\bigl|\Delta_{l,i_{1:l}}\bigr|.
\end{IEEEeqnarray}

Where $\mu^\mathrm{test}_L \sim \mu_l^{i_{1:l}}$ is an independent test point with the same distribution as the leaf nodes rooted in $\mu_l^{i_{1:l}}$.

\end{proposition}

The decomposition in Proposition \ref{prop:gen_bound_decomposition} is an intuitive result that allows for controlling the generalization of 
the algorithm in terms of smaller "local" generalization errors measuring the difference in performance of parent-child nodes. 

\subsection{Supersample construction}

We now introduce a supersample construction in the hierarchical FL setup, following the technique of \cite{steinkeReasoningGeneralizationConditional}.
The idea is to construct a "ghost" dataset $\bar \mu$ with the same statistical properties as the original dataset $\mu$, and measure the generalization
error in terms of an information measure between the algorithm output hypothesis when training on $\mu$ or $\bar \mu$.
At the root layer $l=1$ we sample $\tilde \mu_1^{i_1} = \left( \mu_{l,1}^{i_1}, \mu_{l,2}^{i_1}\right) \sim \mathcal{D}^{\otimes 2} $. 
For each node of the tree we sample $U^{i_{1:l}}_l$ uniformly at random from $\{1, 2\}$ (so one decision variable per tree node) 
and sample the rest of the supersample as follows:
\begin{equation}
    \tilde \mu_l^{i_{1:l}} = \left( \mu_{l, 1}^{i_{1:l}}, \mu_{l, 2}^{i_{1:l}} \right) \sim \left( \mu_{l-1, U^{i_1:l}_{l-1}}^{i_{1:l-1}} \right)^{\otimes 2}.
\end{equation}
In this way we have the following conditional independence:
\begin{equation}
    \mu_{l,1}^{i_{1:l}} \perp \mu_{l,2}^{j_{1:l}} \mid \mu_{l-1, U^{i_1:l}_{l-1}}^{i_{1:l-1}},
\end{equation}
and the following distributional equivalence:
\begin{equation}
    \mu_{l,1}^{i_{1:l}} \overset{d}{=} \mu_{l,2}^{i_{1:l}} \mid \mu_{l-1, U^{i_1:l}_{l-1}}^{i_{1:l-1}}.
\end{equation}

We denote $\bar U_l^{i_{1:l}}$ to be the flipped version of $U_l^{i_{1:l}}$, i.e., $\bar U_l^{i_{1:l}} = 3 - U_l^{i_{1:l}}$.
We also denote by $\mu_l^{i_{1:l}} = \mu_{l, U^{i_1:l}_l}^{i_{1:l}}$ the selected node at layer $l$ and by $\bar \mu_l^{i_{1:l}} = \mu_{l, \bar U^{i_1:l}_l}^{i_{1:l}}$ the ghost node at layer $l$.

\section{Main results}
\label{section:main_results}
In this section we present our main results. We start with the Wasserstein distance bound, together with some implications and discussion, and then move
to the differential privacy bound. All the proofs can be found in the appendix.

\subsection{Wasserstein distance bounds}

\begin{assumption}[Lipschitz property]
    \label{assumption:loss_Lipschitz}
    We assume that the loss function $\ell(w, z)$ is $L$-Lipschitz in its first argument for all $z \in \mathcal{Z}$, i.e., for all $w, w' \in \mathcal{W}$ and $z \in \mathcal{Z}$,
    \begin{equation}
        |\ell(w, z) - \ell(w', z)| \leq L \rho(w, w').
    \end{equation}
\end{assumption}

\begin{theorem}
    \label{thm:abstract_gen_bound_wasserstein}
    Under the above hierachical sampling structure and supersample construction, and under the Lipschitz Assumption \ref{assumption:loss_Lipschitz} on the loss, the following holds:
        \begin{IEEEeqnarray}{rCl}
            \mathbb{E}_{W,\mu_L}\,\operatorname{gen}(W,\mu_L)
            &\leq&
            2L_{\mathrm{Lip}} \sum_{l=1}^{L} \frac{1}{N_l} \sum_{i_{1:l}}
            \Biggl|
            \mathbb{E}_{\tilde{\mu}_l^{i_{1:l}},\,U_l^{i_{1:l}}}
            \Biggr.
            \nonumber\\
            &&\Biggl.
            \mathbb{W}\!\left(
            P_{W \mid \tilde{\mu}_l^{i_{1:l}},\,U_l^{i_{1:l}}},
            P_{W \mid \tilde{\mu}_l^{i_{1:l}}}
            \right)
            \Biggr|,
            \label{eq:abstract_gen_bound_wasserstein}
        \end{IEEEeqnarray}

        where the Wasserstein distance is computed with respect to the same metric $\rho$ that appears in the Lipschitz assumption.
\end{theorem}
\begin{proof}
    The proof is based on the decomposition of Propoisiton \ref{prop:gen_bound_decomposition} and can be found in the Appendix \ref{app:proofs_main_results}.
\end{proof}

\begin{corollary}
    \label{corollary:mutual_information_bound}
    For bounded losses in $[0, 1]$, we obtain:
    \begin{equation}
        \mathbb{E}_{W, \mu_L} gen(W, \mu_L) \leq \sum_{l=1}^L \frac{1}{N_l} \sum_{i_{1:l}}  \E_{\tilde \mu_l^{i_{1:l}}} \sqrt{ 2 I(W; U_l^{i_{1:l}} | \tilde \mu_l^{i_{1:l}})}.
    \end{equation}
\end{corollary}

This is our main result, which relates the generalization error with the Wasserstein distance between the distribution of the output hypothesis when conditioning on a single selected node and its ghost copy. 

We observe two contrasting effects in the bound: on the one hand, the number of terms scales as $O(L)$, on the other hand, the more layers we have, the lower the intra-class variance
will be at each parent-child node, which should reduce the Wasserstein distance between the two distributions.
We also note that the first layers are the most impactful, which is reasonable since they introduce correlations among very large chunks of the dataset and it is an
inherent feature of the hierarchical sampling strategy.

\begin{lemma}
    \label{lemma:Lipschitz_condition_expected_value}
    Under the Lipschitz Assumption \ref{assumption:loss_Lipschitz}, the function $f: \mathcal{W} \times \mathcal{P}(\mathcal{Z}) \to \R$ defined as
    $f(w, P) = \E_{z \sim P}[\ell(w, z)]$ is $L$-Lipschitz in its first argument for any distribution $P$ over $\mathcal{Z}$.
\end{lemma}

\begin{remark}
    In the general case, Lemma \ref{lemma:Lipschitz_condition_expected_value} is tight  (for example when $P$ is a Dirac delta distribution)
     but it introduces looseness whenever $P$ is diffuse.
    In general our hypothesis $W \in \mathcal{W}$ is the output of a learning algorithm that minimizes the empirical risk on the global distribution
    $\mu_L \sim \mu_{L-1} \sim \cdots \sim \mu_1 \sim D$, so it is reasonable to expect that different realizations of $W$ will all perform somewhat similarly
    when averaging.
\end{remark}

\begin{corollary}
    \label{corollary:wang_implication}
    Under the Assumption that the loss is bounded in $[0, 1]$, Theorem \ref{thm:abstract_gen_bound_wasserstein} implies 
    \cite[Theorem 3]{wangGeneralizationFederatedLearning2025}.
\end{corollary}

\begin{corollary}
    \label{corollary:kavian_implication}
    Under the Assumption that the loss is bounded, and that the sampling is positionally dependent (as described in 
    Remark \ref{remark:positionally_dependant_sampling}), Theorem \ref{thm:abstract_gen_bound_wasserstein} implies \cite[Theorem 1]{kavianHeterogeneityMattersEven2025}.
\end{corollary}

Those bounds are in terms of mutual information, hence the need to assume bounded losses to apply Pinsker's inequality with the discrete metric. Both proofs contain long chains of inequalities, which introduce looseness.

\begin{lemma}
    \label{lemma:bound_with_different_conditioning}
    Under the same Assumptions of Theorem \ref{thm:abstract_gen_bound_wasserstein}, by following the steps in the proof of the participation gap
    for Corollary \ref{corollary:wang_implication} we can get the following bound:

    \begin{IEEEeqnarray}{rCl}
        \mathbb{E}_{W,\mu_L}\,&&\operatorname{gen}(W,\mu_L)
        \leq
        2L \sum_{l=1}^{L} \frac{1}{N_l} \sum_{i_{1:l}}
        \Biggl|
        \mathbb{E}_{\tilde{\mu}_l^{i_{1:l}},\,U_l^{i_{1:l}}}
        \mathbb{E}_{\mu_L^{i_{1:l}} \sim \mu_l^{i_{1:l}}}
        \Biggr.
        \nonumber\\
        &&\Biggl.
        \mathbb{E}_{\bar{\mu}_L^{i_{1:l}} \sim \bar{\mu}_l^{i_{1:l}}}
        \mathbb{W}\!\left(
        P_{W \mid \tilde{\mu}_L^{i_{1:l}},\,U_l^{i_{1:l}}},
        P_{W \mid \tilde{\mu}_L^{i_{1:l}}}
        \right)
        \Biggr|,
    \end{IEEEeqnarray}
    where the outer expecation is as usual over the realization of the supersample at layer $l$ and index $i_{1:l}$, but the inner expectations are over the whole realization of the tree under the node $\mu_l^{i_{1:l}}$ and its ghost copy $\bar \mu_l^{i_{1:l}}$.

\end{lemma}

In general it is not possible to compare the two bounds given the lack of data processing inequality for the Wasserstein distance. However, 
in the case of a bounded loss, after applying Pinsker and moving the expectation inside the square root using Jensen we have that Lemma \ref{lemma:bound_with_different_conditioning} 
is looser than Corollary \ref{corollary:mutual_information_bound}. We quantify the looseness in the proof in the appendix.

\subsection{Privacy Bound}
\label{section:privacy}

We now show that our framework is also suitable for deriving generalization bounds from differential privacy (DP) assumptions.
We consider a \textit{single round} setting, where each node forms an hypothesis $W_l^{i_{1:l}}$ by aggregating the hypothesis formed at the lower layer $l+1$ 
(in the case of the last layer, those hypothesis would just be the data points). We can then encode the privacy of the whole FL algorithm by imposing local privacy constraints
on the aggregation steps as follows:
\begin{assumption}
    \label{assumption:privacy}
    Each aggregation algorithm $P_{W^{i_{1:l-1}}_{l-1}|W_{l}^{i_{1:l-1}}}$ is $\epsilon_l$-differentially private, i.e.:
    \begin{equation}
        \frac{P_{W^{i_{1:l-1}}_{l-1}|W_{l}^{i_{1:l-1}}}(S|w)}{P_{W^{i_{1:l-1}}_{l-1}|W_{l}^{i_{1:l-1}}}(S|w')} \leq e^{\epsilon_l}.
    \end{equation}
    For any $w, w' \in \mathcal{W}$ that are differing in a single data point.
    This also implies $I(W^{i_{1:l-1}}_{l-1}; W_{l}^{i_{1:l}}) \leq \epsilon_l (e^{\epsilon_l} - 1)$.
\end{assumption}
Intuitively, this condition implies that at any point in the tree, the output hypothesis of the parent does not reveal information about
the output hypothesis of the children, and therefore about the data points.
\begin{theorem}
    \label{thm:generalization_bound_privacy_assumption}
    Under the privacy Assumption \ref{assumption:privacy}, loss bounded in $[0, 1]$, and the same setup as Theorem \ref{thm:abstract_gen_bound_wasserstein}, the following holds:
    \begin{equation}
        \mathbb{E}_{W, \mu_L} gen(W, \mu_L) \leq 2 \sum_{l=1}^L  \sqrt{ \min( \epsilon_l, \epsilon_l (e^{\epsilon_l} - 1) ) }.
    \end{equation}
\end{theorem}

\section{Gaussian Location Model}
\label{section:gaussian_location_model}

We now analyze the effectiveness of our bounds for the Gaussian location model, where data points are sampled hierarchically and are therefore correlated. 
This particular model allows for exact computation of the generalization error, which allows us to compare our bounds with the correct baseline.

For simplicity, given that variances are fixed, we will identify Gaussian distributions with their mean, which will ease the notation. 

We consider a sampling tree of depth $L$ with root parameter $\theta$. At the first layer, we draw $n_1$ independent samples
\begin{equation}
    \mu_1^{i} \sim \mathcal{N}\!\bigl(\theta,\sigma_1^2\bigr), \qquad i=1,\dots,n_1,
\end{equation}
and, recursively, at layer $\ell=2,\dots,L$, each node at layer $\ell-1$ spawns $n_\ell$ children sampled according to
\begin{equation}
    \mu_\ell^{i_1,\dots,i_\ell} \sim \mathcal{N}\!\bigl(\mu_{\ell-1}^{i_1,\dots,i_{\ell-1}},\sigma_\ell^2\bigr),
    \qquad i_\ell=1,\dots,n_\ell.
\end{equation}

The hypothesis $W$ is the average of all the leaf nodes, which can be written as:
\begin{equation}
    W = \frac{1}{N_L} \sum_{i_{1:L}} \mu_L^{i_{1:L}} = \theta + \sum_{l=1}^L \frac{1}{N_l} \sum_{i_{1:l}} \epsilon_l^{i_{1:l}},
\end{equation}
where $\epsilon_l^{i_{1:l}} \sim \mathcal{N}(0, \sigma_l^2)$ are independent noise terms.

In this setting, we have that the true generalization error is given by:
\begin{equation}    
    \mathbb{E}_{W, \mu_L} gen(W, \mu_L) = \sqrt{\frac{2}{\pi}} \left( \sqrt{V + \Delta} - \sqrt{V - \Delta} \right),
\end{equation}
where $N_l = \prod_{k=1}^l n_k$, $V = \sum_l \sigma_l^2$ and $\Delta = \sum_l \frac{\sigma_l^2}{N_l}$.
Expanding with Taylor series around $\Delta = 0$ we get:
\begin{align}
    \mathbb{E}_{W, \mu_L} gen(W, \mu_L) &= \sqrt{\frac{2}{\pi}} \left( \frac{\Delta}{\sqrt{V}} + O \left(\frac{\Delta^3}{V^{5/2}}\right) \right) \\  
    &\approx \sqrt{\frac{2}{\pi}} \left( \frac{\sum_l \frac{\sigma_l^2}{N_l}}{\sqrt{\sum_l \sigma_l^2}} \right).
\end{align}
 
While, computing the Wasserstein bound with Theorem \ref{thm:abstract_gen_bound_wasserstein} gives:
\begin{equation}
    \mathbb{E}_{W, \mu_L} gen(W, \mu_L) \leq \frac{2}{\sqrt{\pi}} \sum_{l=1}^L \frac{\sigma_l}{N_l}.
\end{equation}

We start by observing that in $L=1$ case, our bound coincides with the true generalization error to 
first order, up to a factor of $\sqrt{2}$.

Adding more layers introduces a larger gap, due to $\sum_{l=1}^L \sigma_l \ge \sum_{l=1}^L \sigma_l^2 / \sqrt{\sum_{l=1}^L \sigma_l^2}$ 
becoming looser.

In particular we also gain an insight on the scaling: The two bounds are the closest when the variance of a single layer dominates.

Specifically, in the homogeneous $\sigma_l = \sigma$ case, we get that the true generalization error is given by:
\begin{equation}
    \mathbb{E}_{W, \mu_L} gen(W, \mu_L) \approx \sqrt{\frac{2}{\pi}} \left( \frac{\sigma^2 \sum_l \frac{1}{N_l}}{\sqrt{L} \sigma} \right ) = \frac{ \sqrt{2} \sigma}{\sqrt{\pi L}} \sum_{l=1}^L \frac{1}{N_l},
\end{equation}
while our bound gives:
\begin{equation}
    \label{eq:GLM_wasserstein}
    \mathbb{E}_{W, \mu_L} gen(W, \mu_L) \leq \frac{2 \sigma }{\sqrt{\pi}}  \sum_{l=1}^L \frac{1}{N_l}.
\end{equation}

Which means that the bound qualitatively captures the scaling with the number of data points but misses the scaling with the number of layers.
For this specific model, the bound from Lemma \ref{lemma:bound_with_different_conditioning} yields
the same expression as Equation \ref{eq:GLM_wasserstein}.


\begin{figure}[htbp]
  \centering
  \includegraphics[width=0.5\textwidth]{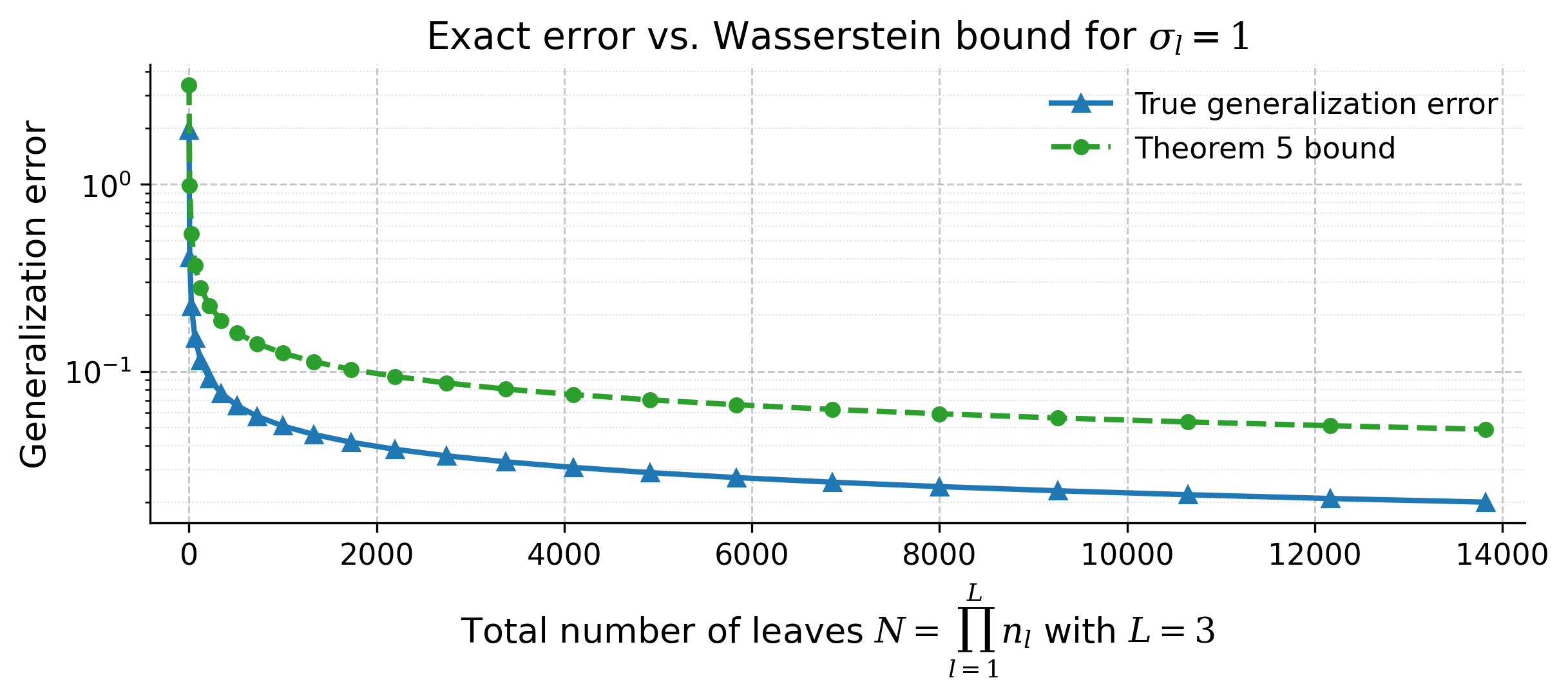}
  \caption{Comparison between the true generalization error and our bound from Theorem \ref{thm:abstract_gen_bound_wasserstein} in the Gaussian Location model.}
  \label{fig:sim}
\end{figure}

\section{Discussion}
\label{section:discussion}

We have presented an abstract formulation of the Hierarchical Federated Learning setup that encloses multiple existing frameworks.
We then derive a generalization bound in terms of the Wasserstein distance by assuming a Lipschitz condition on the loss function and by using a supersample construction.
This bound implies as special cases existing CMI bounds in the literature, but it is more general and can be applied to any hierarchical sampling procedure.
We also compared the bound with the true generalization error in the Gaussian Location Model, showing that it captures the scaling with the number of data points, 
but misses the scaling with the depth of the sampling tree. We also showed that the framework is suitable for deriving generalization bounds from privacy assumptions, and we recovered
a generalization bound based on a local differential privacy assumption on the aggregation steps.
A possible future direction is to extend the framework to general tree structures, in order to enlarge the sampling procedures covered by our bound.
We acknowledge that our bound does not take into consideration the communication structure and/or computational limitations of each node. We addressed this problem
when proving the implication in Corollary \ref{corollary:kavian_implication}, showing that adding independence conditions to model privacy is straightforward
and can be further explored.

\bibliography{MyLibrary,extrarefs}

\newpage
\onecolumn
\appendices

\section{Extra preliminaries and definition}
\label{app:preliminaries}
In this section we breifly overview some of the definitions and results that we will use for the proofs.
\subsection{Information theory}

\begin{definition}[KL divergence]
    Given two probability distributions $P$ and $Q$ over the same space $\mathcal{X}$, the Kullback-Leibler divergence from $Q$ to $P$ is defined as
    \begin{equation}
        D_{KL}(P || Q) = \int_{\mathcal{X}} \log\left( \frac{dP}{dQ}(x) \right) dP(x),
    \end{equation}
    where $\frac{dP}{dQ}$ is the Radon-Nikodym derivative of $P$ with respect to $Q$.    
\end{definition}

\begin{definition}[Mutual information]
    Given two random variables $X$ and $Y$, the mutual information between them is defined as
    \begin{equation}
        I(X; Y) = D_{KL}(P_{XY} || P_X \otimes P_Y),
    \end{equation}
    where $D_{KL}$ is the Kullback-Leibler divergence.
\end{definition}

\begin{definition}[Conditional mutual information]
    Given three random variables $X$, $Y$ and $Z$, the conditional mutual information between $X$ and $Y$ given $Z$ is defined as
    \begin{equation}
        I(X; Y | Z) = \E_{Z} [ D_{KL}(P_{XY|Z} || P_{X|Z} \otimes P_{Y|Z}) ].
    \end{equation}
\end{definition}

\begin{definition}
    [Total variation distance]
    Given two probability distributions $P$ and $Q$ over the same space $\mathcal{X}$, the total variation distance between $P$ and $Q$ is defined as
    \begin{equation}
        \| P - Q \|_{TV} = \sup_{A \subseteq \mathcal{X}} |P(A) - Q(A)| = \frac{1}{2} \int_{\mathcal{X}} |dP(x) - dQ(x)|.
    \end{equation}
\end{definition}

\begin{theorem}[Pinsker's inequality]
    For any two probability distributions $P$ and $Q$ over the same space $\mathcal{X}$, the total variation distance between $P$ and $Q$ is bounded by the square root of the KL divergence:
    \begin{equation}
        \| P - Q \|_{TV} \leq \sqrt{\frac{1}{2} D_{KL}(P || Q)}.
    \end{equation}
\end{theorem}

\begin{definition}
    [Differential privacy]
    A randomized algorithm $\mathcal{A}: \mathcal{D} \to \mathcal{W}$ is said to be $(\epsilon, \delta)$-differentially private if for any two neighboring datasets $D$ and $D'$ differing in at most one element, and for any output set $S \subseteq \mathcal{W}$,
    \begin{equation}
        \Pr[\mathcal{A}(D) \in S] \leq e^{\epsilon} \Pr[\mathcal{A}(D') \in S] + \delta.
    \end{equation}
\end{definition}

\subsection{Wasserstein Distance}
\begin{definition}
    [Wasserstein distance as optimal transport cost]
    Given two probability distributions $P$ and $Q$ over a metric space $(\mathcal{X}, \rho)$, the Wasserstein distance of order 1 between $P$ and $Q$ is defined as
    \begin{equation}
        \mathbb{W}(P, Q) = \inf_{\pi \in \Pi(P, Q)} \E_{(X, Y) \sim \pi} [\rho(X, Y)],
    \end{equation}
    where $\Pi(P, Q)$ is the set of all couplings of $P$ and $Q$, i.e., the set of all joint distributions $\pi$ on $\mathcal{X} \times \mathcal{X}$ such that the marginals of $\pi$ are $P$ and $Q$ respectively. 
\end{definition}

\begin{theorem}
    [Kantorovich-Rubinstein duality]
    The Wasserstein distance can also be expressed in dual form as
    \begin{equation}
        \mathbb{W}(P, Q) = \sup_{f: \text{Lip}(f) \leq 1} \left[ \E_{X \sim P}[f(X)] - \E_{Y \sim Q}[f(Y)] \right],
    \end{equation}
    where the supremum is taken over all 1-Lipschitz functions $f: \mathcal{X} \to \R$.
\end{theorem}

\begin{theorem}
    [Wasserstein distance and total variation distance]
    When the metric of the underlying metric space $(\mathcal{X}, \rho)$ is the discrete metric, i.e., $\rho(x, y) = 1$ if $x \neq y$ and $\rho(x, x) = 0$, then the Wasserstein distance between two probability distributions $P$ and $Q$ is equal to the total variation distance between $P$ and $Q$:
    \begin{equation}
        \mathbb{W}(P, Q) = \| P - Q \|_{TV}.
    \end{equation}
\end{theorem}

\section{Proofs of the main results}
\label{app:proofs_main_results}

\begin{proof}[Proof of Proposition~\ref{prop:gen_bound_decomposition}]
    Proof of Proposition \ref{prop:gen_bound_decomposition}.
    Starting from the definition of the generalization error \eqref{eq:gen_error_definition}, we can add and subtract the following quantity:
    \begin{equation}
        \sum_{l=1}^{L-1} \left( \frac{1}{N_l} \sum_{i_l^{1:l}} \E_{W, \mu} \E_{\mu^\text{test}_L \sim \mu^{i_1:l}_{l}} \left[ \ell(W, \mu^\text{test}_L) \right] \right),
    \end{equation}
    we then collect the terms telescopically to get the following: 
    \begin{align}
        \text{gen}(P_{W|\mu}) &= \sum_{l=0}^{L-1} \frac{1}{N_l} \left[ \sum_{i_l^{1:l}} \E_{W, \mu} \E_{\mu^\text{test}_L \sim \mu^{i_1:l}_{l}} \left[ \ell(W, \mu^\text{test}_L) \right] - \frac{1}{n_{l+1}} \sum_{i_{l+1}^{1:l+1}} \E_{W, \mu} \E_{\mu^\text{test}_L \sim \mu_{l+1}^{i_{1:l+1}}} \left[ \ell(W, \mu^\text{test}_L) \right]  \right] \\
        &\leq \sum_{l=0}^{L-1} \frac{1}{N_l} \sum_{i_l^{1:l}} \left| \E_{W, \mu} \E_{\mu^\text{test}_L \sim \mu^{i_1:l}_{l}} \left[ \ell(W, \mu^\text{test}_L) \right] - \frac{1}{n_{l+1}} \sum_{i_{l+1}^{l+1}} \E_{W, \mu} \E_{\mu^\text{test}_L \sim \mu_{l+1}^{i_{1:l+1}}} \left[ \ell(W, \mu^\text{test}_L) \right]  \right|, \label{eq:proof_appendix_error_decomp}
    \end{align}
    where we applied the triangle inequality to get the desired result.

    The reason for the assumption of uniform branching factor is that it allows to split the error contribution of each branch 
    equally. Continuing from equation \eqref{eq:proof_appendix_error_decomp}, we then introduce the supersample and get a clean 1-to-1 sum of contributions, while if we had an unbalanced tree we would
    have differently scaled loss functions.
\end{proof}

\begin{proof}[    Proof of Lemma \ref{lemma:Lipschitz_condition_expected_value}]
    For all $w, w' \in \mathcal{W}$ and $P \in \mathcal{P}(\mathcal{Z})$, we have:
    \begin{align}
        |f(w, P) - f(w', P)| &= |\E_{z \sim P}[\ell(w, z)] - \E_{z \sim P}[\ell(w', z)]| \\
        &\leq \E_{z \sim P}[|\ell(w, z) - \ell(w', z)|] \\
        &\leq L \rho(w, w'),
    \end{align}
    where we applied the triangle inequality and then the Lipschitz assumption on the loss function.
\end{proof}

\begin{proof}[Proof of theorem \ref{thm:abstract_gen_bound_wasserstein}]
    Starting from the decomposition in Proposition \ref{prop:gen_bound_decomposition}, we can rewrite it in terms of the supersample:
    \begin{equation}
        |\text{gen}(P_{W|\mu}) |  \leq \sum_{l=0}^{L-1} \frac{1}{N_l} \sum_{i_l^{1:l}} \left| \E_{W, \tilde\mu, U} \E_{\mu^\text{test}_L \sim \mu^{i_1:l}_{l, U_l^{i_1:l}}} \left[ \ell(W, \mu^\text{test}_L) \right] - \frac{1}{n_{l+1}} \sum_{i_{l+1}^{l+1}} \E_{W, \tilde\mu, U} \E_{\mu^\text{test}_L \sim \mu_{l+1, U_{l+1}^{i_{1:l+1}}}^{i_{1:l+1}}} \left[ \ell(W, \mu^\text{test}_L) \right]  \right|
    \end{equation}
    by observing that $\mu^\text{test}_L \sim \mu^{i_1:l}_{l, U_l^{i_1:l}}$ has the same distribution as $\mu^\text{test}_L \sim \mu_{l+1, \bar U_{l+1}^{i_{1:l+1}}}^{i_{1:l+1}}$ and both are independent of $W$.
    We then get:
    \begin{align}
        |\text{gen}(P_{W|\mu}) |  &\leq \sum_{l=0}^{L-1} \frac{1}{N_l} \sum_{i_l^{1:l}} \left| \frac{1}{n_{l+1}} \sum_{i_{l+1}^{l+1}} \E_{W, \tilde\mu, U} \E_{\mu^\text{test}_L \sim \mu_{l+1, \bar U_{l+1}^{i_{1:l+1}}}^{i_{1:l+1}}} \left[ \ell(W, \mu^\text{test}_L) \right] - \frac{1}{n_{l+1}} \sum_{i_{l+1}^{l+1}} \E_{W, \tilde\mu, U} \E_{\mu^\text{test}_L \sim \mu_{l+1, U_{l+1}^{i_{1:l+1}}}^{i_{1:l+1}}} \left[ \ell(W, \mu^\text{test}_L) \right]  \right| \\
        &\leq \sum_{l=0}^{L-1} \frac{1}{N_{l+1}} \sum_{i_{l+1}^{1:l+1}} \left| \E_{W, \tilde\mu, U} \E_{\mu^\text{test}_L \sim \mu_{l+1, \bar U_{l+1}^{i_{1:l+1}}}^{i_{1:l+1}}} \left[ \ell(W, \mu^\text{test}_L) \right] - \E_{W, \tilde\mu, U} \E_{\mu^\text{test}_L \sim \mu_{l+1, U_{l+1}^{i_{1:l+1}}}^{i_{1:l+1}}} \left[ \ell(W, \mu^\text{test}_L) \right]  \right| \label{eq:gen_bound_decomposition_abstract_proof}.
    \end{align}
    Now we can add the following quantity in each term of the sum:
    \begin{equation}
        \label{eq:w_prime_abstract_proof_main_eq}
        \E_{W', \tilde\mu, U} \left[ \E_{\mu^\text{test}_L \sim \mu_{l+1, \bar U_{l+1}^{i_{1:l+1}}}^{i_{1:l+1}}} \left[ \E_{W' \sim P_{W|\tilde\mu, U}}[\ell(W', \mu^\text{test}_L)] \right] - \E_{\mu^\text{test}_L \sim \mu_{l+1, U_{l+1}^{i_{1:l+1}}}^{i_{1:l+1}}} \left[ \E_{W' \sim P_{W|\tilde\mu, U}}[\ell(W', \mu^\text{test}_L)] \right]       \right],
    \end{equation}
    where $W'$ is a copy of $W$ that is independent of $W$ and $U_{l+1}^{i_{1:l+1}}$ but dependent of $\tilde\mu$ so $W' \sim P_{W|\tilde\mu} \otimes P_{U_{l+1}^{i_{1:l+1}}}$.
    Therefore, $W'$ is a mixture of an hypothesis trained with $U_{l+1}^{i_{1:l+1}} = 1$ and an hypothesis trained with $U_{l+1}^{i_{1:l+1}} = 2$,
    so Equation \eqref{eq:w_prime_abstract_proof_main_eq} is equal to zero. 
    Combining with Equation \eqref{eq:gen_bound_decomposition_abstract_proof}, for each $l$ and $i_{l+1}^{i_{1:l+1}}$ we get the following two terms:

    \begin{equation}
        \label{eq:proof_abstract_generalization_test_term}
       \E_{W, W', \tilde\mu, U} \E_{\mu^\text{test}_L \sim \mu_{l+1, \bar U_{l+1}^{i_{1:l+1}}}^{i_{1:l+1}}} \left[ \ell(W, \mu^\text{test}_L) - \ell(W', \mu^\text{test}_L) \right],
    \end{equation}
    \begin{equation}
        \label{eq:proof_abstract_generalization_train_term}
       \E_{W, W', \tilde\mu, U} \E_{\mu^\text{test}_L \sim \mu_{l+1, U_{l+1}^{i_{1:l+1}}}^{i_{1:l+1}}} \left[ \ell(W', \mu^\text{test}_L) - \ell(W, \mu^\text{test}_L) \right].
    \end{equation}

    The only difference is that the term in Equation \eqref{eq:proof_abstract_generalization_test_term} is the expected value using the test node of the supersample, while the term in \eqref{eq:proof_abstract_generalization_train_term} uses the train node.
    By Lemma \ref{lemma:Lipschitz_condition_expected_value} the function inside equations \eqref{eq:proof_abstract_generalization_test_term} and \eqref{eq:proof_abstract_generalization_train_term} is $L$-Lipschitz in its first argument, so we can change the order of the expected value and apply 
    the Kantorovich-Rubinstein duality to get:
    \begin{align}
        \label{eq:proof_wasserstein_abstract_before_kr_duality}
        \E_{\tilde\mu, U} \left[ \E_{W} \E_{\mu^\text{test}_L \sim \mu_{l+1, \bar U_{l+1}^{i_{1:l+1}}}^{i_{1:l+1}}} \ell(W, \mu^\text{test}_L) - \E_{W'}\E_{\mu^\text{test}_L \sim \mu_{l+1, \bar U_{l+1}^{i_{1:l+1}}}^{i_{1:l+1}}} \ell(W', \mu^\text{test}_L) \right] &\leq \\
        L \E_{\tilde\mu_{l+1}^{i_{1:l+1}}, U_{l+1}^{i_{1:l+1}}} \left[ \mathbb{W} \left( P_{W|\tilde\mu_{l+1}^{i_{1:l+1}}, U_{l+1}^{i_{1:l+1}}}, P_{W|\tilde\mu_{l+1}^{i_{1:l+1}}} \right) \right] &.
    \end{align}
    For Equation \eqref{eq:proof_abstract_generalization_train_term} we can perform the same steps and get the same bound, which then substituting into \eqref{eq:gen_bound_decomposition_abstract_proof} gives:
    \begin{equation}
        |\text{gen}(P_{W|\mu}) |  \leq 2L \sum_{l=0}^{L-1} \frac{1}{N_{l+1}} \sum_{i_{l+1}^{1:l+1}} \left| \mathbb{E}_{\tilde \mu_{l+1}^{i_{1:l+1}}, U_{l+1}^{i_{1:l+1}}} \mathbb{W} \left( P_{W|\tilde \mu_{l+1}^{i_{1:l+1}}, U_{l+1}^{i_{1:l+1}}}, P_{W|\tilde \mu_{l+1}^{i_{1:l+1}}} \right) \right|
    \end{equation}
    and, up to shifting all the indices by one, we get the desired result.
\end{proof}

\begin{proposition}
    \label{prop:conditioning_increases_wasserstein}
    Given i.i.d. variables $X_i \sim P_X$ and $W \sim P_{W|X_{1:n}}$, calling $X = (X_1, \ldots, X_n)$ then the following holds (under regularity assumptions): 
    \begin{equation}
        \mathbb{E}_{X_i } \mathbb{W}(P_{W|X_i}, P_W) \leq \mathbb{E}_{X} \mathbb{W}(P_{W|X}, P_W) \text{ for all } i \in [n],
    \end{equation}
    which implies:
    \begin{equation}
        \frac{1}{n} \sum_{i=1}^n \mathbb{E}_{X_i } \mathbb{W}(P_{W|X_i}, P_W) \leq \mathbb{E}_{X} \mathbb{W}(P_{W|X}, P_W).
    \end{equation}
\end{proposition}
\begin{proof}
    This Proposition is a restatement of Proposition 2 of \cite{rodriguez-galvezTighterExpectedGeneralization2022} (Section D.1.1).
\end{proof}

\begin{proof}[Proof of corollary \ref{corollary:wang_implication}]
    The decomposition in \cite{wangGeneralizationFederatedLearning2025} is equivalent to our setting when $L=2$, $n_1 = n$, $n_2 = K$. In their notation
    the supersample is denoted by $\tilde Z$ and indexed using two control variables $V_i$ (equivalent to our $U_1^{i}$) and $U^{i, V_i}_{j}$ (equivalent to our $U_2^{i, j}$).

    Their supersample has one extra index because they average the information metric over the whole ghost subtree, not only the intermediate node.
    We split the proof for the \textbf{participation gap} and the \textbf{out-of-sample gap}.

    \textbf{Out-of-sample gap:} We begin by applying the definition of total variation distance and the Pinsker inequality to get:
    \begin{align}
            \sum_{i=1}^{K} \sum_{j=1}^n  \frac{2}{nK} \left| \mathbb{E}_{\tilde \mu_2^{i, j}, U_2^{i, j}} \mathbb{W} \left( P_{W|\tilde \mu_2^{i, j}, U_2^{i, j}}, P_{W|\tilde \mu_2^{i, j}} \right) \right| &\leq \\
            \frac{1}{nK} \sum_{i=1}^{K} \sum_{j=1}^n \E_{\tilde \mu_2^{i, j}, U_2^{i, j}} \sqrt{ 2 D_{KL} \left( P_{W|\tilde \mu_2^{i, j}, U_2^{i, j}} \| P_{W|\tilde \mu_2^{i, j}} \right) } &\leq \\
            \frac{1}{nK} \sum_{i=1}^{K} \sum_{j=1}^n \E_{\tilde \mu_2^{i, j}} \sqrt{ 2 I(W, U_2^{i,j} \mid \tilde \mu_2^{i, j}) }.
    \end{align}

    Now we observe that $\tilde \mu_2^{i, j}$ corresponds to the pair $(\tilde Z ^ i , V_i)$ of \cite{wangGeneralizationFederatedLearning2025}.

    \textbf{Participation gap:} 
    Starting from Equation \eqref{eq:proof_wasserstein_abstract_before_kr_duality} we want to wrap everything in the 
    realization of an extra ghost sample. In our setting we only have the ghost of the same layer sampling, but we do not
    have the whole realization of the path going through the $\bar U$ branch.
    In their notation, the pair $(\tilde Z^i, U_i)$ corresponds to two realization of the end leaves, one per each branch controlled by $V_i$.
    Then, we can apply the law of total expectation to the generalization gap to get:
    
    \begin{align}
        \E_{\tilde\mu, U} \left[ \E_{W} \E_{\mu^\text{test}_L \sim \mu_{l, \bar U_{l}^{i_{1:l}}}^{i_{1:l}}} \ell(W, \mu^\text{test}_L) - \E_{W'}\E_{\mu^\text{test}_L \sim \mu_{l, \bar U_{l}^{i_{1:l}}}^{i_{1:l}}} \ell(W', \mu^\text{test}_L) \right] &= \\
        \E_{\tilde\mu_l^{i_{1:l}}, U_l^{i_{1:l}}} \E_{\mu_L^{i_{1:l}} \sim \tilde\mu_{l, U_{l}^{i_{1:l}}}^{i_{1:l}}} \E_{\bar\mu_L^{i_{1:l}} \sim \tilde\mu_{l, \bar U_{l}^{i_{1:l}}}^{i_{1:l}}} \left[ \E_{W} \E_{\mu^\text{test}_L \sim \mu_{l, \bar U_{l}^{i_{1:l}}}^{i_{1:l}}} \ell(W, \mu^\text{test}_L) - \E_{W'}\E_{\mu^\text{test}_L \sim \mu_{l, \bar U_{l}^{i_{1:l}}}^{i_{1:l}}} \ell(W', \mu^\text{test}_L) \right],
    \end{align}
    where $W$ and $W'$ are trained respectively on $\mu_L^{i_{1:l}}$ and $\bar\mu_L^{i_{1:l}}$, which are the two branches coming from the supersample at layer $l$ up to the last layer $L$.
    We can then apply the Kantorovich-Rubinstein duality to get:
    \begin{align}
        \E_{\tilde\mu_l^{i_{1:l}}, U_l^{i_{1:l}}} \E_{\mu_L^{i_{1:l}} \sim \tilde\mu_{l, U_{l}^{i_{1:l}}}^{i_{1:l}}} \E_{\bar\mu_L^{i_{1:l}} \sim \tilde\mu_{l, \bar U_{l}^{i_{1:l}}}^{i_{1:l}}} \left[ \E_{W} \E_{\mu^\text{test}_L \sim \mu_{l, \bar U_{l}^{i_{1:l}}}^{i_{1:l}}} \ell(W, \mu^\text{test}_L) - \E_{W'}\E_{\mu^\text{test}_L \sim \mu_{l, \bar U_{l}^{i_{1:l}}}^{i_{1:l}}} \ell(W', \mu^\text{test}_L) \right] &\leq \\
        L \E_{\tilde\mu_l^{i_{1:l}}, U_l^{i_{1:l}}} \E_{\mu_L^{i_{1:l}} \sim \tilde\mu_{l, U_{l}^{i_{1:l}}}^{i_{1:l}}} \E_{\bar\mu_L^{i_{1:l}} \sim \tilde\mu_{l, \bar U_{l}^{i_{1:l}}}^{i_{1:l}}} \left[ \mathbb{W} \left( P_{W \mid \tilde\mu_l^{i_{1:l}}, U_l^{i_{1:l}}, \mu_L^{i_{1:l}}, \bar\mu_L^{i_{1:l}}}, P_{W \mid \tilde\mu_l^{i_{1:l}}, \mu_L^{i_{1:l}}, \bar\mu_L^{i_{1:l}}} \right) \right] &= \\
        L \E_{\tilde\mu_l^{i_{1:l}}, U_l^{i_{1:l}}} \E_{\mu_L^{i_{1:l}} \sim \tilde\mu_{l, U_{l}^{i_{1:l}}}^{i_{1:l}}} \E_{\bar\mu_L^{i_{1:l}} \sim \tilde\mu_{l, \bar U_{l}^{i_{1:l}}}^{i_{1:l}}} \left[ \mathbb{W} \left( P_{W \mid U_l^{i_{1:l}}, \mu_L^{i_{1:l}}, \bar\mu_L^{i_{1:l}}}, P_{W \mid \mu_L^{i_{1:l}}, \bar\mu_L^{i_{1:l}}} \right) \right], \\
    \end{align}
    where the last equation holds because of Proposition \ref{prop:conditioning_increases_wasserstein}.
    Replacing the notation from \cite{wangGeneralizationFederatedLearning2025}, specializing the result to $l=1$, $L=2$ and adding the other equal term we get:
    \begin{align}
        2L \E_{\tilde\mu_1^{i}, V_i} \E_{(\tilde Z^i, U_i)} \left[ \mathbb{W} \left( P_{W \mid V_i, (\tilde Z^i, U_i)}, P_{W \mid (\tilde Z^i, U_i)} \right) \right] &\leq \\
        2L \E_{V_i} \E_{(\tilde Z^i, U_i)} \left[ \mathbb{W} \left( P_{W \mid V_i, (\tilde Z^i, U_i)}, P_{W \mid (\tilde Z^i, U_i)} \right) \right] &\leq \\
        \E_{(\tilde Z^i, U_i)} \left[ \sqrt{ 2 I(W, V_i \mid \tilde Z^i, U_i) } \right],
    \end{align}
    which is the desired result.
\end{proof}

\begin{proof}[Proof of corollary \ref{corollary:kavian_implication}]
    Assume that at the first layer the sampling of $\mu_1^i$ depends on $i$ and is deterministic. Up to propagating this index down the tree,
    both the setting and wasserstein bound still hold, and, when $L=2$, we can recover the same setting of \cite{kavianHeterogeneityMattersEven2025}.
    We now show that our bound is strictly tighter. First we observe that at $l=1$ the generalization error is zero, as there is no randomness in the sampling.
    Passing to their notation we have that our bound
    \begin{equation}
        \frac{2L}{N_2} \sum_{k} \sum_{j} \E_{\tilde \mu_2^{k, j}, U_2^{k, j}} \left[ \mathbb{W} \left( P_{W|\tilde \mu_2^{k, j}, U_2^{k, j}}, P_{W|\tilde \mu_2^{k, j}} \right) \right]
    \end{equation}
    can be written in their notation as follows:
    \begin{equation}
        \frac{2 \text{ Lipschitz Constant}}{Kn} \sum_{k=1}^K \sum_{j=1}^n \E \left[ \mathbb{W} \left( P_{W|\tilde Z_{k, j}, J_{k,j}}, P_{W|\tilde Z_{k, j}} \right) \right].
    \end{equation}
    Assuming that the loss is bounded in $[0, 1]$ we can use the discrete metric and the Pinsker inequality (as done in the proof of Corollary \ref{corollary:wang_implication}) to get:
    \begin{align}
        \frac{1}{Kn} \sum_{k=1}^K \sum_{j=1}^n \E \left[ \sqrt{ 2 D_{KL} \left( P_{W|\tilde Z_{k, j}, J_{k,j}} \| P_{W|\tilde Z_{k, j}} \right) } \right] &\leq \\
        \frac{1}{Kn} \sum_{k=1}^K \sum_{j=1}^n \E \left[ \sqrt{ 2 I(W, J_{k,j} \mid \tilde Z_{k, j}) } \right] &\leq \\
        \frac{1}{K} \sum_{k=1}^K \sqrt{ \frac{2}{n} \sum_{j=1}^n I(W, J_{k,j} \mid \tilde Z_{k, j}) },
    \end{align}
    where we have applied Jensen's inequality twice and the definition of conditional mutual information.
    We now want to drop the dependence on $j$ in the mutual information term, which can be done in two steps:
    \begin{align}
        I(W, J_{k,j} \mid \tilde Z_{k, j}) &= H(J_{k,j} \mid \tilde Z_{k, j}) - H(J_{k,j} \mid W, \tilde Z_{k, j}) \\
        &\leq H(J_{k,j} \mid \tilde Z_{k}) - H(J_{k,j} \mid W, \tilde Z_{k}) \\
        &= I(W, J_{k,j} \mid \tilde Z_{k}),
    \end{align}
    which holds because $J_{k,j}$ is an independent Bernoulli variable and by the data processing inequality.
    Similarly:
    \begin{align}
        \sum_{j=1}^n I(W, J_{k,j} \mid \tilde Z_{k}) &= \sum_{j=1}^n H(J_{k,j} \mid \tilde Z_{k}) - H(J_{k,j} \mid W, \tilde Z_{k}) \\
        &\leq \sum_{j=1}^n H(J_{k, j} \mid \tilde Z_{k}, J_{k, 1:j-1}) - H(J_{k,j} \mid W, \tilde Z_{k}, J_{k, 1:j-1}) \\
        &= H(J_{k} \mid \tilde Z_{k}) - H(J_{k} \mid W, \tilde Z_{k}) \\ 
        &= I(W, J_{k} \mid \tilde Z_{k}),
    \end{align}
    where we have applied the chain rule for entropy and the data processing inequality.
    Substituting back we get:
    \begin{align}
        \frac{1}{K} \sum_{k=1}^K \sqrt{ \frac{2}{n} \sum_{j=1}^n I(W, J_{k,j} \mid \tilde Z_{k, j}) } &\leq \frac{1}{K} \sum_{k=1}^K \sqrt{ \frac{2}{n} I(W, J_{k} \mid \tilde Z_{k}) } \\
        &\leq \sqrt{ \frac{2}{nK} \sum_{k=1}^K I(W, J_{k} \mid \tilde Z_{k}) } \\
        &\leq \sqrt{ \frac{2}{nK} \sum_{k=1}^K I(W_k, J_{k} \mid \tilde Z_{k}) }.
    \end{align}
\end{proof}

\begin{proof}[Proof of the ordering of lemma \ref{lemma:bound_with_different_conditioning}]
    For simplicity, call $X = \tilde\mu_l^{i_{1:l}}$, $Y = (\mu_L^{i_{1:l}}, \bar\mu_L^{i_{1:l}})$ and $U = U_l^{i_{1:l}}$. We have that $U$ is independent of $X$ and $Y$ and that 
    the following Markov chain holds: $X - (Y, U) - W$. 
    We start by expanding the following quantities:
    \begin{align}
        I(U;W,Y|X) &= I(U;Y|X) + I(U;W|X,Y) = I(U;W|X,Y)\\
                   &= I(U;W|X) + I(U;Y|X,W), \\
    \end{align}
    and
    \begin{align}
        I(U;W,X|Y) &= I(U;X|Y) + I(U;W|X,Y) = I(U;W|X,Y) \\
                   &= I(U;W|Y) + I(U;X|Y,W). \\
    \end{align}
    We have used the independence of $U$ and the chain rule for the conditional mutual information to obtain the equalities. We notice that we have two expression for $I(U;W|X,Y)$, so we can write:
    \begin{equation}
        I(U;W|X) + I(U;Y|X,W) = I(U;W|Y) + I(U;X|Y,W).
    \end{equation}
    We observe that $I(U;X|Y,W)=0$ because of the following factorization:
    \begin{equation}
        P(U, X | Y, W) = \frac{P(W|Y, U, X) P(Y|U,X) P(U)P(X)}{P(W,Y)} = \frac{P(W|Y, U) P(Y|X) P(U)P(X)}{P(W,Y)} ,
    \end{equation}
    which shows that $U$ and $X$ are independent given $Y$ and $W$. Therefore we get:
    \begin{equation}
        I(U;W|X) + I(U;Y|X,W) = I(U;W|Y),
    \end{equation}
    and since the CMI is non-negative we get the desired result:
    \begin{equation}
        I(W; U_l^{i_{1:l}} | \tilde\mu_l^{i_{1:l}}) \leq I(W; U_l^{i_{1:l}} | \mu_L^{i_{1:l}}, \bar\mu_L^{i_{1:l}}) .
    \end{equation}
    The last step is the only inequality we have used, therefore the relaxation of the two bounds only differs by $I(U;Y|X,W)$.
\end{proof}

\section{Computations for the Gaussian Location Model}

    We start by observing that:
    \begin{equation}
        \mu_L^{i_{1:L}} = \theta + \sum_{l=1}^L Z_l^{i_{1:l}}.
    \end{equation}
    Where $Z_l^{i_{1:l}} \sim \mathcal{N}(0, \sigma^2_l)$ are independent Gaussian noise variables.
    The hypothesis is given by $W = \theta + \sum_{l=1}^L \frac{1}{N_l} \sum_{i_l} Z_l^{i_{1:l}}$.
    The test point is an independent leaf node sampled as $\mu^\text{test}_L \sim \mathcal{N}(\theta, \sum_l^L \sigma^2_l)$.
    The population risk is then simply $| W - \mu^\text{test}_L|$, which is the absolute value of a Gaussian variable with variance:
    \begin{equation}
        \text{Var}(W - \mu^\text{test}_L) = \text{Var}(W) + \text{Var}(\mu^\text{test}_L) = V + \Delta,
    \end{equation}
    where $\Delta = \sum_{l=1}^L \frac{\sigma^2_l}{N_l}$ and $V = \sum_{l=1}^L \sigma^2_l$.
    We then get that $\E[ |L(W, \mu^\text{test}_L) | ] = \sqrt{ \frac{2}{\pi} } \sqrt{\Delta + V}$.
    For the empirical risk we need to compute the expectation of $|W - \mu_L^{i_{1:L}}|$. The inner part is a Gaussian variable and we can compute by observing that:
    \begin{equation}
        W - \mu_L^{i_{1:L}} = \sum_l^L \left( \bar Z_l - Z_l^{i_{1:l}} \right),
    \end{equation}
    where $\bar Z_l = \frac{1}{N_l} \sum_{i_l} Z_l^{i_{1:l}}$.
    Then:
    \begin{align}
        \text{Var}(\bar Z_l - Z_l^{i_{1:l}}) &= \text{Var}(\bar Z_l) + \text{Var}(Z_l^{i_{1:l}}) - 2 \text{Cov}(\bar Z_l, Z_l^{i_{1:l}}) \\
        &= \frac{\sigma^2_l}{N_l} + \sigma^2_l - 2 \frac{\sigma^2_l}{N_l} = \sigma^2_l - \frac{\sigma^2_l}{N_l}.
    \end{align}
    Therefore by summing over $l$ we get $\E[ |\hat L(W, \mu_L^{i_{1:L}})| ] = \sqrt{ \frac{2}{\pi} } \sqrt{V - \Delta}$ and the generalization error is:
    \begin{equation}
        \E gen(W, \mu_L) =\sqrt{ \frac{2}{\pi} } \left( \sqrt{V + \Delta} - \sqrt{V - \Delta} \right).
    \end{equation}

    For the Wasserstein bound, starting from the bound of theorem \ref{thm:abstract_gen_bound_wasserstein} we first notice that our loss 
    function has Lipschitz constant equal to 1.
    We now aim to compute:
    \begin{equation}
        \mathbb{W} \left( P_{W|\tilde \mu_l^{i_{1:l}}, U_l^{i_{1:l}}}, P_{W|\tilde \mu_l^{i_{1:l}}} \right) .
    \end{equation}
    
    Conditional on $\tilde\mu_l^{i_{1:l}}$ and $U_l^{i_{1:l}}=u$ we can write:
    \begin{equation}
        \label{eq:proof_gml_starting_wasserstein}
        W = \frac{1}{N_l} \left( \left( \sum_{i'_{1:l} \neq i_{1:l}} \frac{N_l}{N_L} \sum_{i'_{l+1:L}} \mu_L^{i'_{1:L}} \right) + \frac{N_l}{N_L} \sum_{i'_{l+1:L}} \mu_L^{i_{1:l}, i'_{l+1:L}}  \right),
    \end{equation}
    where $\mu_L^{i_{1:l}, i'_{l+1:L}} \sim \tilde\mu_{l,u}^{i_{1:l}}$ are the leaf nodes of the subtree controlled by $U_l^{i_{1:l}}$.
    Basically, we have taken out of the summation the subtree that is controlled by $U_l^{i_{1:l}}$.
    The first part of the summation does not depend on $u$ and we can just call it $A$. We observe that $W$ is the sum of Gaussian variables, and is therefore a Gaussian. Computing its expected value:
    \begin{equation}
        \E[W | \tilde\mu_l^{i_{1:l}}, U_l^{i_{1:l}}=u] = \E[A] + \frac{1}{N_L} \sum_{i'_{l+1:L}} \E[\mu_L^{i'_{1:L}} | \tilde\mu_{l,u}^{i_{1:l}}] = \E[A] + \frac{\tilde\mu_{l, u}^{i_{1:l}}}{N_l}.
    \end{equation}
    We then notice that $P_{W|\tilde \mu_l^{i_{1:l}}}$ is a mixture of $P_{W|\tilde \mu_l^{i_{1:l}}, U_l^{i_{1:l}}=u}$ for $u=1, 2$, and using that:
    \begin{align}
        \mathbb{W} (\frac{1}{2} P_1 + \frac{1}{2} P_2, P_1) &\leq \frac{1}{2} \mathbb{W}(P_1, P_1) + \frac{1}{2} \mathbb{W}(P_2, P_1) = \frac{1}{2} \mathbb{W}(P_2, P_1), \\
    \end{align}  
    combined with the fact that the Wasserstein distance between two Gaussians is the distance between their means, we get:
    \begin{align}
        \mathbb{W} \left( P_{W|\tilde \mu_l^{i_{1:l}}, U_l^{i_{1:l}}}, P_{W|\tilde \mu_l^{i_{1:l}}} \right) &\leq \frac{1}{2} \left| \E[W | \tilde\mu_l^{i_{1:l}}, U_l^{i_{1:l}}=1] - \E[W | \tilde\mu_l^{i_{1:l}}, U_l^{i_{1:l}}=2] \right| \\
        &= \frac{1}{2} \left| \frac{\tilde\mu_{l, 1}^{i_{1:l}}}{N_l} - \frac{\tilde\mu_{l, 2}^{i_{1:l}}}{N_l} \right| \\
        &= \frac{1}{2 N_l} \left| \tilde\mu_{l, 1}^{i_{1:l}} - \tilde\mu_{l, 2}^{i_{1:l}} \right| .
    \end{align}
    Given that $\tilde\mu_{l, 1}^{i_{1:l}}$ and $\tilde\mu_{l, 2}^{i_{1:l}}$ are independently sampled from their ancestor, they have the same mean and variance equal to $\sigma^2_{l}$,
    so we have $\tilde\mu_{l, 1}^{i_{1:l}} - \tilde\mu_{l, 2}^{i_{1:l}} \sim \mathcal{N}(0, 2\sigma^2_l)$.
    Therefore, taking the expected value we get:
    \begin{equation}
        \E_{\tilde \mu_l^{i_{1:l}}, U_l^{i_{1:l}}} \left[ \mathbb{W} \left( P_{W|\tilde \mu_l^{i_{1:l}}, U_l^{i_{1:l}}}, P_{W|\tilde \mu_l^{i_{1:l}}} \right) \right] \leq \frac{1}{2 N_l} \E \left[ \left| \tilde\mu_{l, 1}^{i_{1:l}} - \tilde\mu_{l, 2}^{i_{1:l}} \right| \right] = \frac{1}{2N_l} \sqrt{ \frac{4\sigma^2_l}{\pi} } = \frac{\sigma_l}{\sqrt{\pi} N_l}.
    \end{equation}

\begin{proof}
    Proof of the equivalence of the bounds using theorem \ref{thm:abstract_gen_bound_wasserstein} and lemma \ref{lemma:bound_with_different_conditioning} for the GLM.
    The two bounds are fundamentally very similar, starting from equation \ref{eq:proof_gml_starting_wasserstein}, we need to assume that the second term in the parenthesis
    is fixed, as we condition on that. We can decompose the terms as sum of each layer noise similarly to the approach for the true generalization error, so we get:
    \begin{align}
        W &= \frac{1}{N_l} \left( \left( \sum_{i'_{1:l} \neq i_{1:l}} \frac{N_l}{N_L} \sum_{i'_{l+1:L}} \mu_L^{i'_{1:L}} \right) + \frac{N_l}{N_L} \sum_{i'_{l+1:L}} \mu_L^{i_{1:l}, i'_{l+1:L}}  \right) \\
        &= \frac{A}{N_L} + \left( \sum_{j_{l+1:L}} \sum_{k=l+1}^L Z_k^{i_{1:l},j_{l+1:L}}  \right) + \frac{\tilde\mu_{l,u}^{i_{1:l}}}{N_l}. \\
    \end{align}
    The law of $W$ is then a Gaussian for both $u=1,2$, and the term in the middle has expecation zero. Therefore, with the same chain of inequalities
    and theorems above, we can get that the Wasserstein distance is upper bounded by:
    \begin{equation}
        \mathbb{W} \le \frac{1}{2 N_l} \left| \tilde\mu_{l, 1}^{i_{1:l}} - \tilde\mu_{l, 2}^{i_{1:l}} \right| ,
    \end{equation}
    which is identical to the bound from theorem \ref{thm:abstract_gen_bound_wasserstein}. We then only have to take the expected value and perform the same computations.
\end{proof}

\section{Differential Privacy Bounds}

\begin{proof}[Proof of theorem \ref{thm:generalization_bound_privacy_assumption}]
    Starting from the mutual information bound from Corollary \ref{corollary:mutual_information_bound}, we derive:
    \begin{align}
        \mathbb{E}_{W, \mu_L} gen(W, \mu_L) \leq 2 &\sum_{l=1}^L \frac{1}{N_l} \sum_{i_{1:l}}  \E_{\tilde \mu_l^{i_{1:l}}} \sqrt{I(W; U_l^{i_{1:l}} | \tilde \mu_l^{i_{1:l}})}  \\
        &\leq 2 \sum_{l=1}^L \frac{1}{N_l}  \sum_{i_{1:l}} \sqrt{ \E_{\tilde \mu_l^{i_{1:l}}} I(W; U_l^{i_{1:l}} | \tilde \mu_l^{i_{1:l}})}  \\
        &\leq 2 \sum_{l=1}^L \frac{1}{N_l}  \sum_{i_{1:l}} \sqrt{ I(W; U_l^{i_{1:l}}, \tilde \mu_l^{i_{1:l}}) - I(W; \tilde \mu_l^{i_{1:l}}) }  \\
        &\leq 2 \sum_{l=1}^L \frac{1}{N_l}  \sum_{i_{1:l}} \sqrt{ I(W; U_l^{i_{1:l}}, \tilde \mu_l^{i_{1:l}}) }  \\
        &\leq 2 \sum_{l=1}^L \frac{1}{N_l}  \sum_{i_{1:l}} \sqrt{ I(W^{i_{1:l-1}}_{l-1}; \mu_l^{i_{1:l}}) }  \\
        &\leq 2 \sum_{l=1}^L \frac{1}{N_l}  \sum_{i_{1:l}} \sqrt{ I(W^{i_{1:l-1}}_{l-1}; W_l^{i_{1:l}}) }  \\
        &\leq 2 \sum_{l=1}^L \frac{1}{N_l}  \sum_{i_{1:l}} \sqrt{ \min( \epsilon_l, \epsilon_l (e^{\epsilon_l} - 1) ) }. 
    \end{align}
    We have applied in order: Jensen's inequality, the definition of CMI, the DPI twice with the Markov property and then the privacy assumption.
    The last step follows from \cite{wangGeneralizationFederatedLearning2025}[Lemma 3.2] which states that $I(W, \mu_i) \leq \min(\epsilon, \epsilon (e^{\epsilon} - 1))$ for an $\epsilon$-differentially private algorithm.
\end{proof}

\end{document}